\documentclass[11pt]{article}
\usepackage[letterpaper]{geometry}
\usepackage{acl2012}
\usepackage{amsmath,amsfonts}
\usepackage{amsthm}
\usepackage{times}
\usepackage{xcolor}
\usepackage{url}
\usepackage{multirow}
\usepackage{hhline}
\usepackage{tikz}

\makeatletter
\newcommand{\@BIBLABEL}{\@emptybiblabel}
\newcommand{\@emptybiblabel}[1]{}
\makeatother
\usepackage[breaklinks,hidelinks]{hyperref}
\usepackage{breakurl}
\usepackage{bookmark}

\usetikzlibrary{matrix}
\usetikzlibrary{decorations.pathreplacing}

\usepackage{everypage}

\usetikzlibrary{calc,tikzmark}

\usepackage{mathtools}
\usepackage{rotating}

\usepackage{color}

\newtheorem{thm}{Theorem}
\newtheorem{lemma}[thm]{Lemma}

\newtheorem{prop}[thm]{Proposition}

\newcommand{\shashicomment}[1]{\textcolor{red}{#1 -- Shashi}}
\newcommand{\vocabsize}{|H|}
\newcommand{\vocab}{H}
\newcommand{\ignore}[1]{}

\newenvironment{itemizesquish}[2]{\begin{list}{\labelitemi}{\setlength{\itemsep}{#1}\setlength{\labelwidth}{#2}\setlength{\leftmargin}{\labelwidth}\addtolength{\leftmargin}{\labelsep}}}{\end{list}}

\title{Encoding Prior Knowledge with Eigenword Embeddings}

\author{
Dominique Osborne \\
Department of Mathematics and Statistics \\
University of Strathclyde \\
Glasgow, G1 1XH, UK \\
{\small \texttt{dominique.osborne.13@uni.strath.ac.uk}}
\And
Shashi Narayan and Shay B. Cohen \\ School of Informatics \\ University of Edinburgh \\ Edinburgh, EH8 9LE, UK \\ \small{\texttt{\{snaraya2,scohen\}@inf.ed.ac.uk}}
}

\mathtoolsset{showonlyrefs,showmanualtags}

\begin{document}

\maketitle

\begin{abstract}
  Canonical correlation analysis (CCA) is a method for reducing the
  dimension of data represented using two views. It has been
  previously used to derive word embeddings, where one view indicates
  a word, and the other view indicates its context. We describe a way
  to incorporate prior knowledge into CCA, give a theoretical
  justification for it, and test it by deriving word embeddings and
  evaluating them on a myriad of datasets.
\end{abstract}

\section{Introduction}

In recent years there has been an immense interest in representing
words as low-dimensional continuous real-vectors, namely word
embeddings. Word embeddings aim to capture lexico-semantic information
such that regularities in the vocabulary are topologically represented
in a Euclidean space. Such word embeddings have achieved
state-of-the-art performance on many natural language processing (NLP)
tasks, e.g., syntactic parsing \cite{socher-13-cvg}, word or phrase
similarity \cite{mikolov-13a
}, dependency parsing
\cite{bansal-14}, unsupervised learning \cite{parikh-14} and others.
Since the discovery that word embeddings are useful as features for
various NLP tasks, research on word embeddings has taken on a life of
its own, with a vibrant community searching for better word
representations in a variety of problems and datasets.

These word embeddings are often induced from large raw text capturing
distributional co-occurrence information via neural networks
\cite{bengio-03,mikolov-13a,mikolov-13b} or spectral methods
\cite{deerwester-90-lsa,dhillon-15}. While these general purpose word
embeddings have achieved significant improvement in various tasks in
NLP, it has been discovered that further tuning of these continuous
word representations for specific tasks improves their performance by
a larger margin. For example, in dependency parsing, word embeddings
could be tailored to capture similarity in terms of context within
syntactic parses \cite{bansal-14} or they could be refined using
semantic lexicons such as WordNet \cite{miller1995wordnet}, FrameNet
\cite{baker-98-framenet} and the Paraphrase Database
\cite{ganitkevitch-13-ppdb} to improve various similarity tasks
\cite{yu-dredze-2014,faruqui:2015:Retro,rothe-15-acl}.  This paper
proposes a method to encode prior semantic knowledge in spectral word
embeddings \cite{dhillon-15}.

Spectral learning algorithms are of great interest for their speed,
scalability, theoretical guarantees and performance in various NLP
applications. These algorithms are no strangers to word embeddings
either. In latent semantic analysis (LSA,
\cite{deerwester-90-lsa,landauer-98}), word embeddings are learned by
performing SVD on the word by document matrix. Recently,
\newcite{dhillon-15} have proposed to use canonical correlation
analysis (CCA) as a method to learn low-dimensional real vectors,
called Eigenwords. Unlike LSA based methods, CCA based methods are
scale invariant and can capture multiview information such as the left
and right contexts of the words. As a result, the eigenword embeddings
of \newcite{dhillon-15} that were learned using the simple linear
methods give accuracies comparable to or better than state of the art
when compared with highly non-linear deep learning based approaches
\cite{collobert-08,mnih-07,mikolov-13a,mikolov-13b}.


The main contribution of this paper is a technique to incorporate
prior knowledge into the derivation of canonical correlation
analysis. In contrast to previous work where prior knowledge is
introduced in the off-the-shelf embeddings as a post-processing step
\cite{faruqui:2015:Retro,rothe-15-acl}, our approach introduces
prior knowledge in the CCA derivation itself. In this way it preserves
the theoretical properties of spectral learning algorithms for
learning word embeddings. The prior knowledge is based on lexical
resources such as WordNet, FrameNet and the Paraphrase Database.

Our derivation of CCA to incorporate prior knowledge is not limited to
eigenwords and can be used with CCA for other problems. It follows a
similar idea to the one proposed by \newcite{koren2003visualization}
for improving the visualization of principal vectors with principal
component analysis (PCA). Our derivation represents the solution to
CCA as that of an optimization problem which maximizes the distance
between the two view projections of training examples, while weighting
these distances using the external source of prior knowledge.
As such, our approach applies to other uses of CCA in the NLP literature,
such as the one of \newcite{jagarlamudi2012regularized}, who used CCA
for transliteration, or the one of \newcite{silberer2013models}, who used
CCA for semantically representing visual attributes.



\section{Background and Notation}
\label{section:background}

For an integer $n$, we denote by $[n]$ the set of integers $\{
1,\ldots,n\}$.  We assume the existence of a vocabulary of words,
usually taken from a corpus. This set of words is denoted by $H = \{
h_1, \ldots, h_{\vocabsize} \}$. For a square matrix $A$, we denote by
$\mathrm{diag}(A)$ a diagonal matrix $B$ which has the same dimensions as
$A$ such that $B_{ii} = A_{ii}$ for all $i$.  For vector $v \in
\mathbb{R}^d$, we denote its $\ell_2$ norm by $||v||$, i.e. $||v|| =
\sqrt{\sum_{i=1}^d v_i^2}$.  We also denote by $v_j$ or $[v]_j$ the
$j$th coordinate of $v$. For a pair of vectors $u$ and $v$, we denote their
dot product by $\langle u, v \rangle$.

We define a word embedding as a function $f$ from $H$ to
$\mathbb{R}^m$ for some (relatively small) $m$.  For example, in our
experiments we vary $m$ between $50$ and $300$.  The word embedding
function maps the word to some real-vector representation, with the
intention to capture regularities in the vocabulary that are
topologically represented in the corresponding Euclidean space. For
example, all vocabulary words that correspond to city names could be
grouped together in that space.

Research on the derivation of word embeddings that capture various
regularities has greatly accelerated in recent
years. Various methods used for this purpose range from low-rank
approximations of co-occurrence statistics
\cite{deerwester-90-lsa,dhillon-15} to neural networks jointly learning
a language model \cite{bengio-03,mikolov-13} or models for other NLP tasks
\cite{collobert-08}.

\section{Canonical Correlation Analysis for Deriving Word Embeddings}
\label{section:cca}

One recent approach to derive word embeddings, developed by
\newcite{dhillon-15}, is through the use of canonical correlation
analysis, resulting in so-called ``eigenwords.'' CCA is a technique
for multiview dimensionality reduction. It assumes the existence of
two views for a set of data, similarly to co-training
\cite{yarowsky1995unsupervised,blum1998combining}, and then projects
the data in the two views in a way that maximizes the correlation
between the projected views.

\newcite{dhillon-15} used CCA to derive word embeddings through the
following procedure. They first break each document in a corpus of
documents into $n$ sequences of words of a fixed length $2k+1$, where
$k$ is a window size. For example, if $k=2$, the short document
``Harry Potter has been a best-seller'' would be broken into ``Harry
Potter has been a'' and ``Potter has been a best-seller.'' In each
such sequence, the middle word is identified as a pivot.

This leads to the construction of the following training set from a
set of documents: $\{
(w^{(i)}_1,\ldots,w^{(i)}_k,w^{(i)},w^{(i)}_{k+1},\ldots,w^{(i)}_{2k})
\mid i \in [n] \}$.  With abuse of notation, this is a multiset, as
certain words are expected to appear in certain contexts multiple
times.  Each $w^{(i)}$ is a pivot word, and the rest of the elements
are words in the sequence called ``the context words.''  With this
training set in mind, the two views for CCA are defined as following.

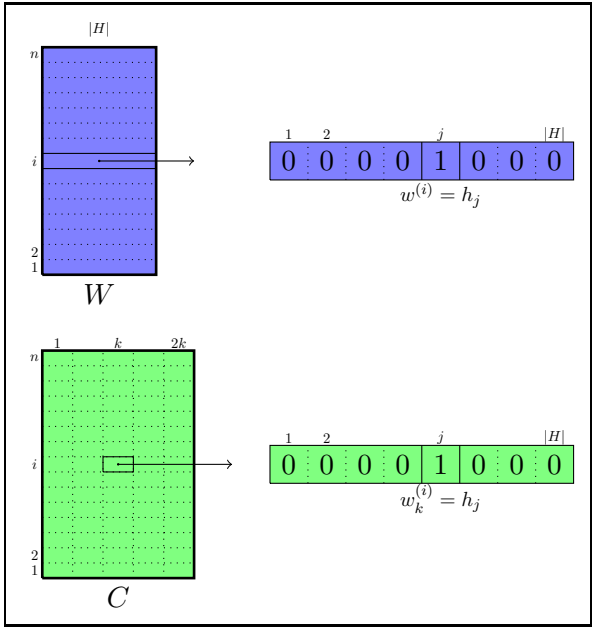
\begin{figure}[th]
  \centering
  \resizebox {\columnwidth} {!} {
\fbox{
    \begin{tikzpicture}
      
      \begin{scope}[xshift=-1cm,yshift=0cm,scale=1]
        \draw [draw=black, line width=0.35mm, fill=blue!50!white] (0,0) -- (1.5,0) -- (1.5,3) -- (0, 3) node[midway,above] {\scalebox{0.5}{$\vocabsize$}} -- (0,0); 
        \draw [draw=black, dotted] (0,0.2) -- (1.5,0.2); 
        \node at (-0.1,0.1) {\scalebox{0.5}{$1$}};
        \draw [draw=black, dotted] (0,0.4) -- (1.5,0.4);
        \node at (-0.1,0.3) {\scalebox{0.5}{$2$}};
        \draw [draw=black, dotted] (0,0.6) -- (1.5,0.6); 
        \draw [draw=black, dotted] (0,0.8) -- (1.5,0.8); 
        \draw [draw=black, dotted] (0,1) -- (1.5,1); 
        \draw [draw=black, dotted] (0,1.2) -- (1.5,1.2); 
        \draw [draw=black] (0,1.4) -- (1.5,1.4);
        \node at (-0.1,1.5) {\scalebox{0.5}{$i$}};
        \draw [draw=black] (0,1.6) -- (1.5,1.6);
        \draw [draw=black, dotted] (0,1.8) -- (1.5,1.8); 
        \draw [draw=black, dotted] (0,2) -- (1.5,2); 
        \draw [draw=black, dotted] (0,2.2) -- (1.5,2.2); 
        \draw [draw=black, dotted] (0,2.4) -- (1.5,2.4); 
        \draw [draw=black, dotted] (0,2.6) -- (1.5,2.6); 
        \draw [draw=black, dotted] (0,2.8) -- (1.5,2.8);
        \node at (-0.1,2.9) {\scalebox{0.5}{$n$}};
        \node at (0.75,-0.25) {$W$};
        \draw [draw=black,->] (0.75,1.5) -- (2,1.5);
        \filldraw (0.75,1.5) circle (0.25pt);
      \end{scope}
      
      \begin{scope}[xshift=2cm,yshift=1.25cm,scale=1]
        \draw [draw=black,fill=blue!50!white] (0,0) -- (4,0) -- (4,0.5) -- (0, 0.5) -- (0,0); 
        \draw [draw=black, dotted] (0.5,0) -- (0.5,0.5);
        \node at (0.25,0.6) {\scalebox{0.5}{$1$}};
        \node at (0.25,0.25) {\scalebox{1}{$0$}};
        \draw [draw=black, dotted] (1,0) -- (1,0.5);
        \node at (0.75,0.6) {\scalebox{0.5}{$2$}};
        \node at (0.75,0.25) {\scalebox{1}{$0$}};
        \draw [draw=black, dotted] (1.5,0) -- (1.5,0.5);
        \node at (1.25,0.25) {\scalebox{1}{$0$}};
        \draw [draw=black] (2,0) -- (2,0.5);
        \node at (1.75,0.25) {\scalebox{1}{$0$}};
        \draw [draw=black] (2.5,0) -- (2.5,0.5);
        \node at (2.25,0.6) {\scalebox{0.5}{$j$}};
        \node at (2.25,-0.2) {\scalebox{0.7}{$w^{(i)} = h_j$}};
        \node at (2.25,0.25) {\scalebox{1}{$1$}};
        \draw [draw=black, dotted] (3,0) -- (3,0.5);
        \node at (2.75,0.25) {\scalebox{1}{$0$}};
        \draw [draw=black, dotted] (3.5,0) -- (3.5,0.5);
        \node at (3.25,0.25) {\scalebox{1}{$0$}};
        \node at (3.75,0.6) {\scalebox{0.5}{$\vocabsize$}};
        \node at (3.75,0.25) {\scalebox{1}{$0$}};
      \end{scope}
      
      \begin{scope}[xshift=-1cm,yshift=-4cm,scale=1]
        \draw [draw=black, line width=0.35mm, fill=green!50!white] (0,0) -- (2,0) -- (2,3) -- (0, 3) 
        -- (0,0); 
        \draw [draw=black, dotted] (0,0.2) -- (2,0.2); 
        \node at (-0.1,0.1) {\scalebox{0.5}{$1$}};
        \draw [draw=black, dotted] (0,0.4) -- (2,0.4);
        \node at (-0.1,0.3) {\scalebox{0.5}{$2$}};
        \draw [draw=black, dotted] (0,0.6) -- (2,0.6); 
        \draw [draw=black, dotted] (0,0.8) -- (2,0.8); 
        \draw [draw=black, dotted] (0,1) -- (2,1); 
        \draw [draw=black, dotted] (0,1.2) -- (2,1.2); 
        \draw [draw=black, dotted] (0,1.4) -- (2,1.4);
        \node at (-0.1,1.5) {\scalebox{0.5}{$i$}};
        \draw [draw=black, dotted] (0,1.6) -- (2,1.6);
        \draw [draw=black, dotted] (0,1.8) -- (2,1.8); 
        \draw [draw=black, dotted] (0,2) -- (2,2); 
        \draw [draw=black, dotted] (0,2.2) -- (2,2.2); 
        \draw [draw=black, dotted] (0,2.4) -- (2,2.4); 
        \draw [draw=black, dotted] (0,2.6) -- (2,2.6); 
        \draw [draw=black, dotted] (0,2.8) -- (2,2.8);
        \node at (-0.1,2.9) {\scalebox{0.5}{$n$}};
        \draw [draw=black, dotted] (0.4,0) -- (0.4,3); 
        \node at (0.2,3.1) {\scalebox{0.5}{$1$}};
        \draw [draw=black, dotted] (0.8,0) -- (0.8,3); 
        \draw [draw=black, dotted] (1.2,0) -- (1.2,3); 
        \node at (1,3.1) {\scalebox{0.5}{$k$}};
        \draw [draw=black, dotted] (1.6,0) -- (1.6,3); 
        \node at (1.8,3.1) {\scalebox{0.5}{$2k$}};
        \draw [draw=black] (0.8,1.4) -- (1.2,1.4) -- (1.2,1.6) -- (0.8, 1.6) -- (0.8,1.4); 
        \draw [draw=black,->] (1,1.5) -- (2.5,1.5);
        \filldraw (1,1.5) circle (0.25pt);
        \node at (1,-0.25) {$C$};
      \end{scope}

      \begin{scope}[xshift=2cm,yshift=-2.75cm,scale=1]
        \draw [draw=black,fill=green!50!white] (0,0) -- (4,0) -- (4,0.5) -- (0, 0.5) -- (0,0); 
        \draw [draw=black, dotted] (0.5,0) -- (0.5,0.5);
        \node at (0.25,0.6) {\scalebox{0.5}{$1$}};
        \node at (0.25,0.25) {\scalebox{1}{$0$}};
        \draw [draw=black, dotted] (1,0) -- (1,0.5);
        \node at (0.75,0.6) {\scalebox{0.5}{$2$}};
        \node at (0.75,0.25) {\scalebox{1}{$0$}};
        \draw [draw=black, dotted] (1.5,0) -- (1.5,0.5);
        \node at (1.25,0.25) {\scalebox{1}{$0$}};
        \draw [draw=black] (2,0) -- (2,0.5);
        \node at (1.75,0.25) {\scalebox{1}{$0$}};
        \draw [draw=black] (2.5,0) -- (2.5,0.5);
        \node at (2.25,0.6) {\scalebox{0.5}{$j$}};
        \node at (2.25,-0.2) {\scalebox{0.7}{$w^{(i)}_k = h_j$}};
        \node at (2.25,0.25) {\scalebox{1}{$1$}};
        \draw [draw=black, dotted] (3,0) -- (3,0.5);
        \node at (2.75,0.25) {\scalebox{1}{$0$}};
        \draw [draw=black, dotted] (3.5,0) -- (3.5,0.5);
        \node at (3.25,0.25) {\scalebox{1}{$0$}};
        \node at (3.75,0.6) {\scalebox{0.5}{$\vocabsize$}};
        \node at (3.75,0.25) {\scalebox{1}{$0$}};
      \end{scope}

    \end{tikzpicture}
  }
}
  \caption{The word and context views represented as matrix $W$ and
    $C$. Each row in $W$ is a vector of length $|H|$, corresponding to
    a one-hot vector for the word in the example indexed by the row.
    Each row in $C$ is a vector of length $2k|H|$, divided into
    sub-vectors each of length $|H|$. Each such sub-vector is a
    one-hot vector for one of the $2k$ context words in the example
    indexed by the row.\label{fig:ccaviews}}
\end{figure}

\ignore{ 
  Formally, we define the first view through a sparse matrix $C \in
  \mathbb{R}^{n \times 2k\vocabsize}$ such that $C_{ij} = 1$ if for
\begin{equation}
  t = \begin{cases} 2k & \mbox{if } j \mod 2k = 0 \\ j \mod 2k & \mbox{otherwise} \end{cases}
\end{equation} 
and for $j' = \left\lceil \displaystyle\frac{j}{v} \right\rceil$, it
holds that $w^{(i)}_t = h_{j'}$.  This just means that $C$ is the
``context matrix'' such that each row in the context matrix is a
vector, consisting of $2k$ one-hot vectors, each of length
$\vocabsize$. Each such one-hot vector corresponds to a word that
fired in a specific index in the context.  In addition, we also define
a second view through a matrix $W \in \mathbb{R}^{n \times
  \vocabsize}$ such that $W_{ij} = 1$ if $w^{(i)} = h_j$.
}

We define the first view through a sparse ``context matrix'' $C \in
\mathbb{R}^{n \times 2k\vocabsize}$ such that each row in the matrix
is a vector, consisting of $2k$ one-hot vectors, each of length
$\vocabsize$. Each such one-hot vector corresponds to a word that
fired in a specific index in the context. In addition, we also define
a second view through a matrix $W \in \mathbb{R}^{n \times
  \vocabsize}$ such that $W_{ij} = 1$ if $w^{(i)} = h_j$. We present
both views of the training set in Figure~\ref{fig:ccaviews}.

Note that now the matrix $M = W^{\top} C$ is in
$\mathbb{R}^{\vocabsize \times (2k\vocabsize)}$ such that each element
$M_{ij}$ gives the count of times that $h_i$ appeared with the
corresponding context word and context index encoded by $j$.

Similarly, we define a matrix $D_1 = \mathrm{diag}(W^{\top}W)$ and
$D_2 = \mathrm{diag}(C^{\top}C)$. Finally, to get the word embeddings,
we perform singular value decomposition (SVD) on the matrix
$D_1^{-1/2} M D_2^{-1/2}$. Note that in its original form, CCA
requires use of $W^{\top}W$ and $C^{\top}C$ in their full form, and not
just the corresponding diagonal matrices $D_1$ and $D_2$; however, in
practice, inverting these matrices can be quite intensive
computationally and can lead to memory issues. As such, we approximate
CCA by using the diagonal matrices $D_1$ and $D_2$.

From the SVD step, we get two projections $U \in
\mathbb{R}^{\vocabsize \times m}$ and $V \in \mathbb{R}^{2k\vocabsize
  \times m}$ such that
\begin{equation}
D_1^{-1/2} M D_2^{-1/2} \approx U \Sigma V^{\top}
\end{equation}

\noindent where $\Sigma \in \mathbb{R}^{m \times m}$ is a diagonal
matrix with $\Sigma_{ii} > 0$ being the $i$th largest singular value
of $D_1^{-1/2} M D_2^{-1/2}$.  In order to get the final word
embeddings, we calculate $D_1^{-1/2} U \in \mathbb{R}^{\vocabsize \times
  m}$. Each row in this matrix corresponds to an $m$-dimensional
vector for the corresponding word in the vocabulary. This means that
$f(h_i)$ for $h_i \in H$ is the $i$th row of the matrix $D_1^{-1/2}
U$.  The projection $V$ can be used to get ``context embeddings.'' See
more about this in \newcite{dhillon-15}.

This use of CCA to derive word embeddings follows the
usual distributional hypothesis \cite{harris1957co} that most word
embeddings techniques rely on.  In the case of CCA, this hypothesis is
translated into action in the following way. CCA finds projections for
the contexts and for the pivot words which are most correlated. This
means that if a word co-occurs in a specific context many times
(either directly, or transitively through similarity to other words),
then this context is expected to be projected to a point ``close'' to
the point to which the word is projected. As such, if two words occur
in a specific context many times, these two words are expected to be
projected to points which are close to each other.

For the next section, we denote $X = W D_1^{-1/2}$ and $Y = C
D_2^{-1/2}$. To refer to the dimensions of $X$ and $Y$ generically, we
denote $d = \vocabsize$ and $d' = 2k\vocabsize$. In addition, we refer
to the column vectors of $U$ and $V$ as $u_1,\ldots,u_m$ and
$v_1,\ldots,v_m$.

\paragraph{Mathematical Intuition Behind CCA} The procedure that CCA
follows finds a projection of the two views in a shared space, such
that the correlation between the two views is maximized at each
coordinate, and there is minimal redundancy between the coordinates of
each view.  This means that CCA solves the following sequence of
optimization problems for $j \in [m]$ where $a_j \in \mathbb{R}^{1
  \times d}$ and $b_j \in \mathbb{R}^{1 \times d'}$:
\begin{align}
\arg\max_{a_j,b_j} &  & \mathrm{corr}(a_j W^{\top}, b_j C^{\top}) \\
\mbox{such that} & & \mathrm{corr}(a_j W^{\top}, a_k W^{\top}) = 0, & & k < j \\
		 &  & \mathrm{corr}(b_j C^{\top}, b_k C^{\top}) = 0, & & k < j
\end{align}
\noindent where $\mathrm{corr}$ is a function that accepts two vectors
and return the Pearson correlation between the pairwise elements of
the two vectors. The approximate solution to this optimization problem
(when using diagonal $D_1$ and $D_2$) is $\hat{a}^{\top}_i =
D_1^{-1/2} u_i$ and $\hat{b}^{\top}_i = D_2^{-1/2} v_i$ for $i \in
[m]$.

CCA also has a probabilistic interpretation as a maximum likelihood
solution of a latent variable model for two normal random vectors,
each drawn based on a third latent Gaussian vector
\cite{bach2005probabilistic}.

The way we describe CCA for deriving word embeddings is related to
Latent Semantic Indexing (LSI), which performs singular value
decomposition on the matrix $M$ directly, without doing any kind of
variance normalization.  \newcite{dhillon-15} describe some
differences between LSI and CCA.  The extra normalization step
decreases the importance of frequent words when doing SVD.

\section{Incorporating Prior Knowledge into Canonical Correlation Analysis}
\label{section:cca-prior-knowledge}

\begin{figure*}
  \centering
  \fbox{
  \resizebox {2\columnwidth} {!} {
  \begin{tikzpicture}

    \begin{scope}[xshift=-3.5cm,yshift=0cm,scale=1]
      \draw [draw=black, line width=0.35mm, fill=blue!50!white] (0,0) -- (1.5,0) -- (1.5,3) -- (0, 3) node[midway,above] {$d$} -- (0,0) node[midway,left] {$n$}; 
      \node at (0.75,1.5) {$W$};
    \end{scope}
    
    \begin{scope}[xshift=1.25cm,yshift=0cm,scale=1]
      \draw[draw=black, line width=0.35mm, fill=red!50!white] (0,0) -- (3,0) -- (3,3) -- (0, 3) node[midway,above] {$n$} -- (0,0) node[midway,left] {$n$}; 
      \node at (1.5,1.5) {$L$};
      \draw[decorate, decoration={brace, amplitude=10pt}] (3.2,-0.2) -- coordinate [below=10pt] (L2) (-0.2,-0.2) node {};
      \draw[-latex,rounded corners=1mm] (L2) -- (1.5, -1) -- (-1.5, -2.3) node[draw=none,fill=none,midway,right] {prior knowledge (optional)} -- (-1.5,-3.3); 
    \end{scope} 

    \begin{scope}[xshift=7.2cm,yshift=0cm,scale=1]
      \draw [draw=black, line width=0.35mm, fill=green!50!white] (0,0) -- (2,0) -- (2,3) -- (0, 3) node[midway,above] {$d'$} -- (0,0) node[midway,left] {$n$}; 
      \node at (1,1.5) {$C$};
    \end{scope}
    
    \begin{scope}[xshift=-6.9cm,yshift=-2.5cm,scale=1]
      \node at (0,-1.5) {$\mbox{diag}$};
      \draw[rounded corners=1mm,thick] (0.6,-2.45) -- (0.45,-2.45) -- (0.45,-0.55) -- (0.6,-0.55);
      \draw [draw=black, line width=0.35mm, fill=blue!50!white] (0.6,-0.75) -- (2.1,-0.75) -- (2.1,-2.25) -- (0.6, -2.25) -- (0.6,-0.75); 
      \draw[rounded corners=1mm,thick] (2.1,-2.45) -- (2.25,-2.45) -- (2.25,-0.55) -- (2.1,-0.55);
      \node at (1.35,-1.5) {$W^{\top}W$};
      \node at (2.5,-0.45) {$^{-\frac{1}{2}}$};
      \draw[decorate, decoration={brace, amplitude=10pt}] (2.3,-2.45) -- coordinate [below=10pt] (D1) (-0.3,-2.45) node {};
      \node at (1,-3.05) {$D_1$};
    \end{scope}

    \begin{scope}[xshift=-3.7cm,yshift=-2.5cm,scale=1]
      \node at (-0.5,-1.5) {$\times$};
      \draw [draw=black, line width=0.35mm, fill=blue!50!white] (0,-0.75) -- (3,-0.75) -- (3,-2.25) -- (0, -2.25) -- (0,-0.75); 
      \node at (1.5,-1.5) {$W^{\top}$};
      \node at (3.5,-1.5) {$\times$};
      \draw [draw=black, line width=0.35mm, fill=green!50!white] (4,0) -- (6,0) -- (6,-3) -- (4, -3) -- (4,0); 
      \node at (5,-1.5) {$C$};
      \node at (6.5,-1.5) {$\times$};
      \draw[decorate, decoration={brace, amplitude=10pt}] (6.1,-3) -- coordinate [below=10pt] (M) (-0.1,-3) node {};
      \node at (3,-3.5) {$M$};
    \end{scope}
    
    \begin{scope}[xshift=3.6cm,yshift=-2.5cm,scale=1]
      \node at (0,-1.5) {$\mbox{diag}$};
      \draw[rounded corners=1mm,thick] (0.6,-2.7) -- (0.45,-2.7) -- (0.45,-0.3) -- (0.6,-0.3);
      \draw [draw=black, line width=0.35mm, fill=green!50!white] (0.6,-0.5) -- (2.6,-0.5) -- (2.6,-2.5) -- (0.6, -2.5) -- (0.6,-0.5); 
      \draw[rounded corners=1mm,thick] (2.6,-2.7) -- (2.75,-2.7) -- (2.75,-0.3) -- (2.6,-0.3);
      \node at (1.6,-1.5) {$C^{\top}C$};
      \node at (3,-0.2) {$^{-\frac{1}{2}}$};
      \draw[decorate, decoration={brace, amplitude=10pt}] (2.8,-2.7) -- coordinate [below=10pt] (D2) (-0.3,-2.7) node {};
      \node at (1.3,-3.3) {$D_2$};
    \end{scope}
    
    
    \begin{scope}[xshift=8cm,yshift=-3.25cm,scale=1]
      \node at (-1,-0.75) {$\approx$};
      \draw [draw=black, line width=0.35mm, fill=blue!50!white] (0,0) -- (0.75,0) node[midway,above] {$m$} -- (0.75,-1.5) -- (0, -1.5)  -- (0,0) node[midway,left] {$d$}; 
      \node at (0.375,-0.75) {$U$};
      \node at (1,-0.75) {$\times$};
      \draw [draw=black, line width=0.35mm, fill=blue!0!white] (1.25,-0.35) -- (2,-0.35) -- (2,-1.1) -- (1.25, -1.1)  -- (1.25,-0.35) -- (2,-1.1); 
      \node at (1.625,-0.75) {$\Sigma$};
      \node at (2.25,-0.75) {$\times$};
      \draw [draw=black, line width=0.35mm, fill=green!50!white] (2.5,-0.35) -- (4.5,-0.35) node[midway,above] {$d'$} -- (4.5,-1.1) node[midway,right] {$m$} -- (2.5, -1.1)  -- (2.5,-0.35) ; 
      \node at (3.5,-0.75) {$V^{\top}$};
    \end{scope}

    \begin{scope}[xshift=-3.7cm,yshift=-3.3cm,scale=1]
      \draw[decorate, decoration={brace, amplitude=10pt}] (3.1,-3) -- coordinate [below=10pt] (M) (-3.5,-3) node {};
      \node at (-0.1,-3.6) {$X^{\top}$};
      \draw[decorate, decoration={brace, amplitude=10pt}] (10.1,-3) -- coordinate [below=10pt] (M) (3.9,-3) node {};
      \node at (7,-3.6) {$Y$};
    \end{scope}

  \end{tikzpicture}
}}
\caption{Introducing prior knowledge in CCA. $W \in \mathbb{R}^{n
    \times d}$ and $C \in \mathbb{R}^{n \times d'}$ denote the word
  and context views respectively. $L \in \mathbb{R}^{n \times n}$ is a
  Laplacian matrix encoded with the prior knowledge about the
  distances between the projections of $W$ and
  $C$.\label{fig:ccawithpk}}
\end{figure*}
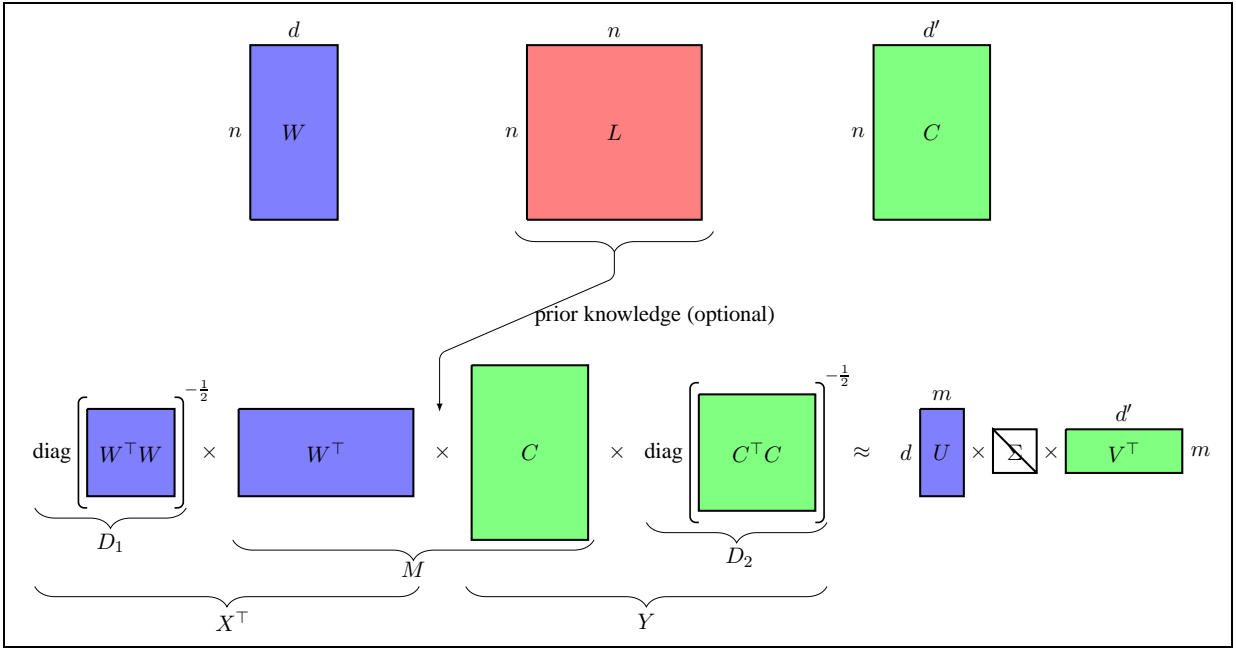

In this section, we detail the technique we use to incorporate prior
knowledge into the derivation of canonical correlation analysis. The main
motivation behind our approach is to improve the optimization of
correlation between the two views by weighing them using the external
source of prior knowledge. The prior knowledge is based on lexical
resources such as WordNet, FrameNet and the Paraphrase Database. Our
approach follows a similar idea to the one proposed by
\newcite{koren2003visualization} for improving the visualization of
principal vectors with principal component analysis (PCA). It is
also related to Laplacian manifold regularization \cite{belkin2006manifold}.

An important notion in our derivation is that of a {\em Laplacian
  matrix}.  The Laplacian of an undirected weighted graph is an $n
\times n$ matrix where $n$ is the number of nodes in the graph. It
equals $D-A$ where $A$ is the adjacency matrix of the graph (so that
$A_{ij}$ is the weight for the edge $(i,j)$ in the graph, if it
exists, and $0$ otherwise) and $D$ is a diagonal matrix such that
$D_{ii} = \sum_j A_{ij}$.  The Laplacian is always a symmetric square
matrix such that the sum over rows (or columns) is $0$.  It is also
positive semi-definite.

We propose a generalization of CCA, in which we introduce a Laplacian
matrix into the derivation of CCA itself, as shown in
Figure~\ref{fig:ccawithpk}. We encode prior knowledge about the
distances between the projections of two views into the Laplacian.
The Laplacian allows us to improve the optimization of the
correlation between the two views by weighing them using the external
source of prior knowledge.

\subsection{Generalization of CCA}

We present three lemmas (proofs are given in Appendix A), followed by
our main proposition. These three lemmas are useful to prove our final
proposition.


The main proposition shows that CCA maximizes the distance between the
two view projections for any pair of examples $i$ and $j$, $i \neq j$,
while minimizing the two view projection distance for the two views of
an example $i$. The two views we discuss here in practice are the view
of the word through a one-hot representation, and the view which
represents the context words for a specific word token. The distance
between two view projections is defined in Eq.~\refeq{eq:dist}.

\begin{lemma}
  Let $X$ and $Y$ be two matrices of size $n \times d$ and $n \times
  d'$, respectively, for example, as defined in
  \S\ref{section:cca}. Assume that $\sum_{i=1}^n X_{ij} = 0$ for $j
  \in [d]$ and $\sum_{i=1}^n Y_{ij} = 0$ for $j \in [d']$.  Let $L$ be
  an $n \times n$ Laplacian matrix such that
  \begin{equation}
    L_{ij} = \begin{cases} n -1 & \mbox{if } i = j \\
      -1 & \mbox{if } i \neq j.
    \end{cases} \label{eq:lu}
  \end{equation}
  Then $X^{\top} L Y$ equals $X^{\top} Y$ up to a multiplication by a
  positive constant.
  \label{lemma1}
\end{lemma}


\begin{lemma}
  Let $A \in \mathbb{R}^{d \times d'}$. Then the rank $m$ thin-SVD of
  $A$ can be found by solving the following optimization problem:
  \begin{align}
    \displaystyle\max_{\begin{array}{ll}u_1,\ldots,u_m, \\v_1,\ldots,v_m\end{array}} & \sum_{i=1}^m u_i^{\top} A v_i \label{eq-obj} \\
    \mbox{such that } & ||u_i|| = ||v_i|| = 1 & i \in [m] \\
    & \langle u_i, u_j \rangle = \langle v_i, v_j \rangle = 0 & i \neq j
  \end{align}
  \noindent where $u_i \in \mathbb{R}^{d \times 1}$ denote the left
  singular vectors, and $v_i \in \mathbb{R}^{d' \times 1}$ denote the
  right singular vectors.
  \label{lemma:svd}
\end{lemma}


The last utility lemma we describe shows that interjecting the
Laplacian between the two views can be expressed as a weighted sum of
the distances between the projections of the two views (these
distances are given in Eq.~\refeq{eq:dist}), where the weights come
from the Laplacian.

\begin{lemma}
  Let $u_1,\ldots,u_m$ and $v_1,\ldots,v_m$ be two sets of vectors of
  length $d$ and $d'$ respectively.  Let $L \in \mathbb{R}^{n \times
    n}$ be a Laplacian and $X \in \mathbb{R}^{n \times d}$ and $Y \in
  \mathbb{R}^{n \times d'}$.  Then:
  \begin{equation}
    \sum_{k=1}^m (X u_k)^{\top} L \left( Y v_k \right) = \sum_{i,j} -L_{ij}\left( d^m_{ij} \right)^2,
  \end{equation}
  \noindent where
  \begin{equation}
    d^m_{ij} = \sqrt{\frac{1}{2}\left( \sum_{k=1}^m \left( [Xu_k]_i - [Yv_k]_j \right)^2 \right)}. \label{eq:dist}
  \end{equation}
  \label{lemma3}
\end{lemma}

The following proposition is our main result for this section.

\begin{prop}
  The matrices $U \in \mathbb{R}^{d \times m}$ and $V \in
  \mathbb{R}^{d' \times m}$ that CCA computes are the $m$-dimensional
  projections that maximize
  \begin{equation}
    \sum_{i, j} \left(d^m_{ij} \right)^2 - n \sum_{i=1}^n \left( d^m_{ii} \right)^2, \label{eq:dist-sum}
  \end{equation}
  \noindent where $d^m_{ij}$ is defined as in Eq.~\refeq{eq:dist} for
  $u_1,\ldots,u_m$ being the columns of $U$ and $v_1,\ldots,v_m$ being
  the columns of $V$.
\label{prop1}
\end{prop}

\begin{proof}
  According to Lemma~\ref{lemma3}, the objective in
  Eq.~\refeq{eq:dist-sum} equals $\sum_{k=1}^m (Xu_k)^{\top} L (Y
  v_k)$ where $L$ is defined as in Eq.~\refeq{eq:lu}.  Therefore,
  maximizing Eq.~\refeq{eq:dist-sum} corresponds to maximization of
  $\sum_{k=1}^m (Xu_k)^{\top} L (Y v_k)$ under the constraints that
  the $U$ and $V$ matrices have orthonormal vectors. Using
  Lemma~\ref{lemma:svd}, it can be shown that the solution to this
  maximization is done by doing singular value decomposition on
  $X^{\top} L Y$. According to Lemma~\ref{lemma1}, this corresponds to
  finding $U$ and $V$ by doing singular value decomposition on
  $X^{\top} Y$, because a multiplicative constant does not change the
  value of the right/left singular vectors.
\end{proof}

The above proposition shows that CCA tries to find projections of both
views such that the distances between the two views for pairs of
examples with indices $i \neq j$ are maximized (first term in
Eq.~\refeq{eq:dist-sum}), while minimizing the distance between the
projections of the two views for a specific example (second term in
Eq.~\refeq{eq:dist-sum}).  Therefore, CCA tries to project a context
and a word in that context to points that are close to each other in a shared
space, while maximizing the distance between a context and a word which
do not often co-occur together.

As long as $L$ is a Laplacian, Proposition~\ref{prop1} is still true, only with the
maximization of the objective


\begin{equation}
  \sum_{i, j} - L_{ij} \left( d^m_{ij} \right)^2, \label{eq:dis}
\end{equation}

\noindent where $L_{ij} \le 0$ for $i \neq j$ and $L_{ii} \ge 0$.
This result lends itself to a
generalization of CCA, in which we use predefined weights for the
Laplacian that encode some prior knowledge about the distances that
the projections of two views should satisfy.

If the weight $- L_{ij}$ is large for a specific $(i,j)$, then we
will try harder to maximize the distance between one view of example
$i$ and the other view of example $j$ (i.e. we will try to project the
word $w^{(i)}$ and the context of example $j$ into distant points in
the space).


This means that in the current formulation, $- L_{ij}$ plays the
role of a {\em dissimiliarity} indicator between pairs of words. The
more dissimilar words are, the larger the weight, and then the
more distant the projections are for the contexts and the words.

\subsection{From CCA with Dissimilarities to CCA with Similarities}

It is often more convenient to work with {\em similarity} measures
between pairs of words. To do that, we can retain the same formulation
as before with the Laplacian, where $- L_{ij}$ now denotes a measure
of similarity. Now, instead of maximizing the objective in
Eq.~\refeq{eq:dis}, we are required to minimize it.

It can be shown that such mirror formulation can be done with an
algorithm similar to CCA, leading to a proposition in the style of
Proposition~\ref{prop1}.  To solve this minimization formulation, we
just need to choose the singular vectors associated with the {\em
  smallest} $m$ singular values (instead of the largest).

Once we change the CCA algorithm with the Laplacian to choose these
projections, we can define $L$, for example, based on a similarity
graph. The graph is an undirected graph that has $\vocabsize$ nodes,
for each word in the vocabulary, and there is an edge between a pair
of words whenever the two words are similar to each other based on
some external source of information, such as WordNet (for example, if
they are synonyms).

We then define the Laplacian $L$ such that $L_{ij} = -1$ if $i$ and
$j$ are adjacent in the graph (and $i \neq j$), $L_{ii}$ is the
degree of the node $i$ and $L_{ij} = 0$ in all other cases. By using
this variant of CCA, we strive to maximize the distance of
the two views between words which are adjacent in the graph (or
continuing the example above, maximize the distance between words
which are not synonyms).  In addition, the fewer adjacent nodes a word
has (or the more synonyms it has), the less important it is to
minimize the distance between the two views of that given word.

\begin{figure}[t!]
{\small
\framebox{\parbox{3in}{

{\bf Inputs:} Set of examples $\{ (w^{(i)}_1,\ldots,w^{(i)}_k,w^{(i)},w^{(i)}_{k+1},\ldots,w^{(i)}_{2k}) \mid i \in [n] \}$, an integer $m$, an $\alpha \in (0,1]$, an undirected graph $G$ over $\vocab$, an integer $N$.

{\bf Data structures:}

A matrix $M$ of size $\vocabsize \times (2k\vocabsize)$ (cross-covariance matrix), a matrix $U$ corresponding to the word embeddings

{\bf Algorithm:}

(Cross-covariance estimation)
$\forall i,j \in [n]$ such that $|i-j| \le N$

\begin{itemizesquish}{-0.3em}{0.5em}

\item If $i = j$, increase $M_{rs}$ by $1$ for $r$ denoting the index of word $w^{(i)}$ and for all $s$ denoting the context indices of words
$w^{(i)}_1,\ldots,w^{(i)}_k$ and $w^{(i)}_{k+1},\ldots,w^{(i)}_{2k}$.

\item  If $i \neq j$ and word $w^{(i)}$ is connected to word $w^{(j)}$ in $G$, increase $M_{rs}$ by $\alpha$ for $r$ denoting the index of word $w^{(i)}$ and for all $s$ denoting the context indices of words
$w^{(j)}_1,\ldots,w^{(j)}_k$ and $w^{(j)}_{k+1},\ldots,w^{(j)}_{2k}$.

$\,$

\item Calculate $D_1$ and $D_2$ as specified in \S\ref{section:cca}.

\end{itemizesquish}

(Singular value decomposition step)
\begin{itemizesquish}{-0.3em}{0.5em}
\item Perform singular value decomposition on $D_1^{-1/2} M D_2^{-1/2}$ to get a matrix $U \in \mathbb{R}^{\vocabsize \times m}$.
\end{itemizesquish}

(Word embedding projection)
\begin{itemizesquish}{-0.3em}{0.5em}
\item For each word $h_i$ for $i \in [\vocabsize]$ return the word embedding that corresponds with the $i$th row of $U$.
\end{itemizesquish}
}}
}
\caption{The CCA-like algorithm that returns word embeddings with
  prior knowledge encoded based on a similarity
  graph.\label{figure:alg}}
\end{figure}

\subsection{Final Algorithm}

In order to use an arbitrary Laplacian matrix with CCA, we require
that the data is centered, i.e. that the average over all examples of
each of the coordinates of the word and context vectors is $0$.
However, such a prerequisite would make the matrices $C$ and $W$ dense
(with many non-zero values), and hard to maintain in memory, and would
also make singular value decomposition inefficient.

As such, we do not center the data to keep it sparse, and as such, use
a matrix $L$ which is not strictly a Laplacian, but that behaves
better in practice.\footnote{We note that other decompositions, such
  as PCA, also require centering of the data, but in case of sparse
  data matrix, this step is not performed.} Given the graph mentioned
in \S\ref{section:cca-prior-knowledge} which is extracted from an
external source of information, we use $L$ such that $L_{ij} = \alpha$
for an $\alpha \in (0,1)$ which is treated as a smoothing factor for
the graph (see below the choices of $\alpha$) if $i$ and $j$ are not
adjacent in the graph, $L_{ij} = 0$ if $i \neq j$ are adjacent, and
finally $L_{ii} = 1$ for all $i \in [n]$.  Therefore, this matrix is
symmetric, and the only constraint it does not satisfy is that of rows
and columns summing to $0$.

Scanning the documents and calculating the statistic matrix with the
Laplacian is computationally infeasible with a large number of tokens
given as input. It is quadratic in that number. As such, we make
another modification to the algorithm, and calculate a ``local''
Laplacian.  The modification requires an integer $N$ as input (we use
$N=12$), and then it makes updates to pairs of word tokens only if
they are within an $N$-sized window of each.  The final algorithm we
use is described in Figure~\ref{figure:alg}. The algorithm works
by directly computing the co-occurrence matrix $M$ (instead of maintaining $W$ and $C$).
It does so by increasing by 1 any cells corresponding to word-context co-occurrence in the documents
and by $\alpha$ any cells corresponding to word and contexts that are connected in the graph.


\section{Experiments}
\label{section:experiments}

In this section we describe our experiments.

\subsection{Experimental Setup}

\paragraph{Training Data}
We used three datasets, {\sc Wiki1}, {\sc Wiki2} and {\sc Wiki5}, all
based on the first $1$, $2$ and $5$ billion words from Wikipedia
respectively.\footnote{We downloaded the data from
  \url{https://dumps.wikimedia.org/}, and preprocessed it using the
  tool available at \url{http://mattmahoney.net/dc/textdata.html}.}
Each dataset is broken into chunks of length $13$ (window sizes of 6),
corresponding to a document. The above Laplacian $L$ is calculated
within each document separately.  This means that $-L_{ij}$ is $1$
only if $i$ and $j$ denote two words that appear in the same
document.
This is done to make the calculations computationally
feasible. We calculate word embeddings for the top most frequent 200K
words.

\paragraph{Prior Knowledge Resources} We consider three sources of
prior knowledge: WordNet \cite{miller1995wordnet}, the Paraphrase
Database of \newcite{ganitkevitch-13-ppdb}, abbreviated as
PPDB,\footnote{We use the XL subset of the PPDB.} and FrameNet
\cite{baker-98-framenet}. Since FrameNet and WordNet index words in
their base form, we use WordNet's stemmer to identify the base form
for the text in our corpora whenever we calculate the Laplacian graph.
For WordNet, we have an edge in the graph if one word is a synonym,
hypernym or hyponym of the other.  For PPDB, we have an edge if one
word is a paraphrase of the other, according to the database. For
FrameNet, we connect two words
in the graph if they appear in the same frame.

\paragraph{System Implementation} We modified the implementation of
the \texttt{SWELL} Java
package\footnote{\url{https://github.com/paramveerdhillon/swell}.} of
\newcite{dhillon-15}.
Specifically, we needed to modify the
loop that iterates over words in each document to a
nested loop that iterates over pairs of words, in order to compute a sum of
the form $\sum_{ij} X_{ri} L_{ij} Y_{js}$.\footnote{Our
  implementation and the word embeddings that we calculated are
  available at
  \url{http://cohort.inf.ed.ac.uk/cohort/eigen/}.  }
\newcite{dhillon-15} use window size $k=2$,
which we retain in our experiments.\footnote{We also use the
  square-root transformation as mentioned in \newcite{dhillon-15}
  which controls the variance in the counts accumulated from the
  corpus. See a justification for this transform in \newcite{stratos2015model}.}

\begin{table*}
  \begin{center}
    {\small
      \begin{tabular}{|l|l|cccc|cccc|cccc|}
	\hhline{~~------------}
        \multicolumn{2}{l|}{} & \multicolumn{4}{c|}{Word similarity average} & \multicolumn{4}{c|}{Geographic analogies} & \multicolumn{4}{c|}{NP bracketing} \\ 
        \multicolumn{2}{l|}{} & NPK & WN & PD & FN & NPK & WN & PD & FN & NPK & WN & PD & FN  \\
        \hline
        \multirow{5}{*}{\rotatebox{91}{Retrofitting}} & \multicolumn{1}{|l|}{Glove} & \tikzmark{A} 59.7  & 63.1 & 64.6 & 57.5   &  \tikzmark{C} {\bf 94.8} & 75.3 & 80.4 & {\bf 94.8} &  \tikzmark{E} 78.1 & 79.5 & 79.4 & 78.7  \\
        & \multicolumn{1}{|l|}{Skip-Gram} &  64.1  &65.5 & {\bf 68.6}  & 62.3 &  87.3 & 72.3 & 70.5 & 87.7 & 79.9 & 80.4 & 81.5 & 80.5  \\
        & \multicolumn{1}{|l|}{Global Context} & 44.4 & 50.0 &50.4  &47.3  & 7.3 & 4.5 & 18.2 & 7.3 & 79.4 & 79.1 & 80.5 & 80.2 \\
        & \multicolumn{1}{|l|}{Multilingual} &62.3  &66.9  &68.2  &62.8  & 70.7 & 46.2 & 53.7 & 72.7 & 81.9 & 81.8 & {\bf 82.7} & 82.0  \\
        & \multicolumn{1}{|l|}{Eigen (CCA)} & 59.5 &62.2  & 63.6 &61.4 \tikzmark{B} & 89.9  & 79.2  & 73.5  & 89.9 \tikzmark{D}  & 81.3 &81.7  &81.2  &80.7 \tikzmark{F}  \\
        \hline
	\hline
       \multirow{5}{*}{\rotatebox{91}{CCAPrior}} & $\alpha = 0.1$ & \tikzmark{G} - &59.1 & 59.6 & 59.5 & \tikzmark{I} - & 88.9 & 88.7 & 89.9 & \tikzmark{K} - &81.0 & {\bf 82.4} &81.0\\
        & ${\alpha = 0.2}$ & - &59.9& {\bf 60.6} &60.0& - & 89.1 & 91.3 & 90.1 &- &81.0 &81.3&80.7\\
        & ${\alpha = 0.5}$ & - &59.9 &59.7 &59.6&- & 86.9 &  89.3 & 89.3 &-&81.8 &81.4&80.9 \\
        & ${\alpha = 0.7}$ & - &60.7  &59.3 &59.5&- & 86.9& 89.3 & 92.9 &- &80.3 &81.2&80.8   \\
        & ${\alpha = 0.9}$ & - &60.6  &59.6 &58.9 \tikzmark{H}  &-& 89.1 & {\bf 93.2} & 92.5 \tikzmark{J}  &-&81.3&80.7&81.0\tikzmark{L}   \\
        \hline
        \multirow{5}{*}{\rotatebox{91}{CCAPrior+RF}} & $\alpha = 0.1$ & \tikzmark{M} - &61.9 &63.6 &61.5&\tikzmark{O} -& 76.0 & 71.9 & 89.9 &\tikzmark{Q}  -&81.4 &81.7&81.2\\
        & $\alpha = 0.2$ & - &62.6 & {\bf 64.9} &61.6&- &78.0  &69.3 & 90.1 &-&81.7 &81.1&80.6 \\
        & $\alpha = 0.5$ & - &62.7 &63.7 &61.4 &-& 74.9& 67.3&92.9 &-& {\bf 81.9} &81.4&80.0\\
        & $\alpha = 0.7$ & - &63.3 &63.0&61.0&-& 77.4& 65.6& 90.3&-&81.0 &80.8&80.4\\
        & $\alpha = 0.9$ & - &62.0 &63.3&60.4\tikzmark{N} &-& 77.3& 66.2 & {\bf 92.5} \tikzmark{P} &-& 81.0&80.7&80.4\tikzmark{R}   \\
        \hline
      \end{tabular}
    }
  \end{center}
  \caption{Results for the word similarity datasets, geographic
    analogies and NP bracketing. The first upper blocks (A--C)
    present the results with retrofitting. NPK stands for no prior
    knowledge (no retrofitting is used), WN for WordNet, PD for PPDB
    and FN for FrameNet.  Glove, Skip-Gram, Global Context,
    Multilingual and Eigen are the word embeddings of
    \protect\newcite{pennington2014glove},
    \protect\newcite{mikolov-13a},
    \protect\newcite{huang2012improving},
    \protect\newcite{faruqui14multi} and \protect\newcite{dhillon-15}
    respectively. The second middle blocks (D--F) show the results of our eigenword
    embeddings encoded with prior knowledge using our method. Each row
    in the block corresponds to a specific use of an $\alpha$ value
    (smoothing factor), as described in Figure~\ref{figure:alg}. In
    the lower blocks (G--I) we take the word embeddings from the
    second block, and retrofit them using the method of
    \protect\newcite{faruqui:2015:Retro}. Best results in each block
    are in bold. }
  \label{table:results}
\end{table*}

\subsection{Baselines} 

\paragraph{Off-the-shelf Word Embeddings}
We compare our word embeddings with existing state-of-the-art word
embeddings, such as Glove \cite{pennington2014glove}, Skip-Gram
\cite{mikolov-13a}, Global Context \cite{huang2012improving} and
Multilingual \cite{faruqui14multi}. We also compare our word
embeddings with the Eigen word embeddings of \newcite{dhillon-15}
without any prior knowledge.

\paragraph{Retrofitting for Prior Knowledge}

We compare our approach of incorporating prior knowledge into the
derivation of CCA against the previous works where prior knowledge is
introduced in the off-the-shelf embeddings as a post-processing step
\cite{faruqui:2015:Retro,rothe-15-acl}. In this paper, we focus on the
retrofitting approach of \newcite{faruqui:2015:Retro}.
 
Retrofitting works by optimizing an objective function which has two
terms: one that tries to keep the distance between the word vectors
close to the original distances, and the other which enforces the
vectors of words which are adjacent in the prior knowledge graph to be
close to each other in the new embedding space. We use the
retrofitting
package\footnote{\url{https://github.com/mfaruqui/retrofitting}.} to
compare our results in different settings against the results of
retrofitting of \newcite{faruqui:2015:Retro}.

\subsection{Evaluation Benchmarks} 

We evaluated the quality of our eigenword embeddings on three
different tasks: word similarity, geographic analogies and NP
bracketing.

\ignore{ We compare our results in different settings against the
  results of retrofitting of \newcite{faruqui:2015:Retro}.
  Retrofitting works by optimizing an objective function which has two
  terms: one that tries to keep the distance between the word vectors
  close to the original distances, and the other which enforces the
  vectors of words which are adjacent in the prior knowledge graph to
  be close to each other in the new embedding space.  }

\AddThispageHook{
\begin{tikzpicture}[overlay,remember picture]
  \node [yshift=.9ex] at ( $(pic cs:A) !.5! (pic cs:B)$ ){ \fontsize{60}{60}\selectfont\textbf{\color{gray!40}A} };
  \node [yshift=.9ex] at ( $(pic cs:C) !.5! (pic cs:D)$ ){ \fontsize{60}{60}\selectfont\textbf{\color{gray!40}B} };
  \node [yshift=.9ex] at ( $(pic cs:E) !.5! (pic cs:F)$ ){ \fontsize{60}{60}\selectfont\textbf{\color{gray!40}C} };
  \node [yshift=.9ex] at ( $(pic cs:G) !.5! (pic cs:H)$ ){ \fontsize{60}{60}\selectfont\textbf{\color{gray!40}D} };
  \node [yshift=.9ex] at ( $(pic cs:I) !.5! (pic cs:J)$ ){ \fontsize{60}{60}\selectfont\textbf{\color{gray!40}E} };
  \node [yshift=.9ex] at ( $(pic cs:K) !.5! (pic cs:L)$ ){ \fontsize{60}{60}\selectfont\textbf{\color{gray!40}F} };
  \node [yshift=.9ex] at ( $(pic cs:M) !.5! (pic cs:N)$ ){ \fontsize{60}{60}\selectfont\textbf{\color{gray!40}G} };
  \node [yshift=.9ex] at ( $(pic cs:O) !.5! (pic cs:P)$ ){ \fontsize{60}{60}\selectfont\textbf{\color{gray!40}H} };
  \node [yshift=.9ex] at ( $(pic cs:Q) !.5! (pic cs:R)$ ){ \fontsize{60}{60}\selectfont\textbf{\color{gray!40}I} };
\end{tikzpicture}
}

\paragraph{Word Similarity} For the word similarity task we
experimented with 11 different widely used benchmarks.
The WS-353-ALL dataset \cite{finkelstein2002placing} consists of
353 pairs of English words with their human similarity ratings. Later,
\newcite{agirre2009study} re-annotated WS-353-ALL for similarity
(WS-353-SIM) and relatedness (WS-353-REL) with specific distinctions
between them. The SimLex-999 dataset
\cite{DBLP:journals/corr/HillRK14} was built to measure how well
models capture similarity, rather than relatedness or association. The
MEN-TR-3000 dataset \cite{bruni2014multimodal} consists of 3000 word
pairs sampled from words that occur at least 700 times in a large web
corpus. The datasets, MTurk-287 \cite{radinsky2011word} and MTurk-771
\cite{halawi2012large}, were scored by Amazon Mechanical Turk workers
for relatedness of English word pairs. The YP-130
\cite{yang2005measuring} and Verb-143 \cite{baker2014unsupervised}
datasets were developed for verb similarity predictions. The last two
datasets, MC-30 \cite{doi:10.1080/01690969108406936} and
RG-65 \cite{Rubenstein:1965:CCS:365628.365657} consist of 30
and 65 noun pairs respectively.


For each dataset, we calculate the cosine similarity between the
vectors of word pairs and measure Spearman's rank correlation
coefficient between the scores produced by the embeddings and human
ratings.  We report the average of the correlations on all 11
datasets.  Each word similarity task in the above list represents a
different aspect of word similarity, and as such, averaging the
results points to the quality of the word embeddings on several
tasks. We later analyze specific datasets.

\paragraph{Geographic Analogies} 

\newcite{mikolov-13b} created a test set of analogous word pairs such
as $a$:$b$ $c$:$d$ raising the analogy question of the form ``$a$ is
to $b$ as $c$ is to \_\_'' where $d$ is unknown. We report results on a subset
of this dataset which focuses on finding capitals of common countries,
e.g., {\it Greece} is to {\it Athens} as {\it Iraq} is to \_\_. This
dataset consists of 506 word pairs. For given word pairs, $a$:$b$
$c$:$d$ where $d$ is unknown, we use the vector offset method
\cite{mikolov-13a}, i.e., we compute a vector $v = v_b - v_a + v_c$
where $v_a$, $v_b$ and $v_c$ are vector representations of the words
$a$, $b$ and $c$ respectively; we then return the word $d$ with the
greatest cosine similarity to $v$.

\paragraph{NP Bracketing} Here the goal is to identify the correct
bracketing of a three-word noun \cite{lazaridou-13}. For example, the
bracketing of {\em annual (price growth)} is ``right,'' while the
bracketing of {\em (entry level) machine} is ``left.'' Similarly to
\newcite{faruqui:2015}, we concatenate the word vectors of the three
words, and use this vector for binary classification into left or
right.

Since most of the datasets that we evaluate on in this paper are not
standardly separated into development and test sets, we report all
results we calculated (with respect to hyperparameter differences) and
do not select just a subset of the results.

\subsection{Evaluation}

\paragraph{Preliminary Experiments} In our first set of experiments,
we vary the dimension of the word embedding vectors. We
try $m \in \{ 50, 100, 200, 300 \}$.  Our experiments
showed that the results consistently improve when the dimension
increases for all the different datasets. For example, for $m=50$ and
{\sc Wiki1}, we get an average of $46.4$ on the word similarity
tasks, $50.1$ for $m=100$, $53.4$ for $m=200$ and $54.2$ for
$m=300$.  The more data are available, the more likely larger dimension
will improve the quality of the word embeddings. Indeed, for {\sc Wiki5},
we get an average of $49.4$, $54.9$, $57.0$ and $59.5$ for
each of the dimensions.  The improvements with respect to
the dimension are consistent across all of our results, so we fix $m$
at $300$. 

We also noticed a consistent improvement in accuracy when
using more data from Wikipedia. For example, for $m=300$, using {\sc
  Wiki1} gives an average of $54.1$, while using {\sc Wiki2} gives an
average of $54.9$ and finally, using {\sc Wiki5} gives an average of
$59.5$. We fix the dataset we use to be {\sc Wiki5}.

\paragraph{Results} Table~\ref{table:results} describes the results
from our first set of experiments. (Note that the table is divided
into 9 distinct blocks, labeled A through I.)  In general, adding prior knowledge
to eigenword embeddings does improve the quality of word vectors for
the word similarity, geographic analogies and NP bracketing tasks on
several occasions (blocks D--F compared to last row in blocks A--C).
For example, our eigenword vectors encoded with
prior knowledge (CCAPrior) consistently perform better than the
eigenword vectors that do not have any prior knowledge for the word
similarity task ($59.5$, Eigen in the first row under NPK
column, versus block D). The only exceptions are for $\alpha=0.1$ with WordNet ($59.1$),
for $\alpha=0.7$ with PPDB ($59.3$) and for $\alpha=0.9$ with FrameNet
($58.9$), where $\alpha$ denotes the smoothing factor.


In several cases, running the retrofitting algorithm of
\newcite{faruqui:2015:Retro} on top of our word embeddings helps
further, as if ``adding prior knowledge twice is better than once.''
Results for these word embeddings (CCAPrior+RF) are shown in Table
\ref{table:results}. Adding retrofitting to our encoding of prior
knowledge often performs better for word similarity and NP bracketing
tasks (block D versus G and block F versus I). Interestingly,
CCAPrior+RF embeddings also often perform better than eigenword
vectors (Eigen) of \newcite{dhillon-15} when retrofitted using the
method of \newcite{faruqui:2015:Retro}. For example, in the word
similarity task, eigenwords retrofitted with WordNet get an accuracy
of $62.2$ whereas encoding prior knowledge using both CCA and
retrofitting gets a maximum accuracy of $63.3$. We see the same
pattern for PPDB, with $63.6$ for ``Eigen'' and $64.9$ for
``CCAPrior+RF''. We hypothesize that the reason for these changes is
that the two methods for encoding prior knowledge maximize different
objective functions.

The performance with FrameNet is weaker, in some cases leading to
worse performance (e.g., with Glove and SG vectors). We believe that
FrameNet does not perform as well as the other lexicons because it
groups words based on very abstract concepts; often words with
seemingly distantly related meanings (e.g., push and growth) can evoke
the same frame. This also supports the findings of
\newcite{faruqui:2015:Retro}, who noticed that the use of FrameNet as
a prior knowledge resource for improving the quality of word
embeddings is not as helpful as other resources such as WordNet and
PPDB.

We note that CCA works especially well for the geographic analogies
dataset. The quality of eigenword embeddings (and the other
embeddings) degrades when we encode prior knowledge using the method
of \newcite{faruqui:2015:Retro}. Our method improves the quality of
eigenword embeddings.




\paragraph{Global Picture of the Results}
When comparing retrofitting to CCA with prior knowledge, there is a noticable
difference. Retrofitting performs well or badly, depending on the dataset, while the
results with CCA are more stable. We attribute this to the difference
between how our algorithm and retrofitting work. Retrofitting makes a {\em direct}
use of the source of prior knowledge, by adding a regularization term that
enforces words which are similar according to the prior knowledge to be
closer in the embedding space. Our algorithm, on the other hand, makes a more
indirect use of the source of prior knowledge, by changing the co-occurence
matrix on which we do singular value decomposition.

Specifically, we believe that our algorithm is more stable to cases
in which words for the task at hand are unknown words with respect to the
source of prior knowledge.
This is demonstrated with the
geographical analogies task: in that case, retrofitting lowers the results in most cases.
The city and country names do not appear in the sources
of prior knowledge we used.

\paragraph{Further Analysis} We further inspected the results on the
word similarity tasks for the RG-65 and WS-353-ALL datasets. Our goal
was to find cases in which either CCA embeddings by themselves
outperform other types of embeddings or that encoding prior knowledge
into CCA the way we describe significantly improves the results.

For the WS-353-ALL dataset, the eigenword embeddings get
a correlation of 69.6. The next best performing word embeddings are the
multilingual word embeddings (68.0) and skip-gram
(58.3). Interestingly enough, the multilingual word embeddings also
use CCA to project words into a low-dimensional space using a linear
transformation, suggesting that linear projections are a good fit for
the WS-353-ALL dataset. The dataset itself includes pairs of common
words with a corresponding similarity score. The words that appear in
the dataset are actually expected to occur in similar contexts, a
property that CCA directly encodes when deriving word embeddings.

The best performance on the RG-65 dataset is with the Glove word
embeddings (76.6). CCA embeddings give an accuracy of 69.7 on that
dataset. However, with this dataset, we observe significant
improvement when encoding prior knowledge using our method. For
example, using WordNet with this dataset improves the results by 4.2
points (73.9).  Using the method of \newcite{faruqui:2015:Retro} (with
WordNet) on top of our CCA word embeddings improves the results even
further by 8.7 points (78.4).

\paragraph{The Role of Prior Knowledge} We also designed an experiment
to test whether using distributional information is necessary for
having well-performing word embeddings, or whether it is sufficient to
rely on the prior knowledge resource. In order to test this, we
created a sparse matrix that corresponds to the graph based on the
external resource graph. We then follow up with singular value
decomposition on that graph, and get embeddings of size
300. Table~\ref{table:svd} gives the results when using these
embeddings. We see that the results are consistently lower than the
results that appear in Table \ref{table:results}, implying that the
use of prior knowledge comes hand in hand with the use of
distributional information.  When using the retrofitting method by
Faruqui et al. on top of these word embeddings, the results barely
improved.

\begin{table}
\begin{center}
{\small
\begin{tabular}{|l|cc|}
\hline
Resource & WordSim  & NP Bracketing  \\
\hline
WordNet & 35.9 & 73.6  \\
PPDB & 37.5  & 77.9  \\
FrameNet & 19.9 & 74.5  \\
\hline
\end{tabular}
}
\end{center}
\caption{Results on word similarity dataset (average over 11 datasets) and NP bracketing. The word embeddings are
  derived by using SVD on the similarity graph extracted from the prior knowledge
  source (WordNet, PPDB and FrameNet).}
\label{table:svd}
\end{table}









\section{Related Work}

Our ideas in this paper for encoding prior knowledge in eigenword
embeddings relate to three main threads in existing literature.

One of the threads focuses on modifying the objective of word vector
training algorithms. \newcite{yu-dredze-2014}, \newcite{Xu-14},
\newcite{fried-14} and \newcite{bian-14} augment the training
objective in neural language models of \newcite{mikolov-13} to
encourage semantically related word vectors to come closer to each
other. \newcite{wang-14} propose a method for jointly embedding
entities (from FreeBase, a large community-curated knowledge base) and
words (from Wikipedia) into the same continuous vector space. 
\newcite{chen2015semantic} propose a similar joint model to improve
the word embeddings, but rather than using structured knowledge
sources their model focuses on discovering stronger semantic
connections in specific contexts in a text corpus.

Another research thread relies on post-processing steps to encode prior
knowledge from semantic lexicons in off-the-shelf word embeddings. The
main intuition behind this trend is to update word vectors by running
belief propagation on a graph extracted from the relation information
in semantic lexicons. The retrofitting approach of \newcite{faruqui:2015:Retro}
uses such techniques to obtain higher quality semantic vectors using
WordNet, FrameNet, and the Paraphrase Database. They report on how
retrofitting helps improve the performance of various
off-the-shelf word vectors such as Glove, Skip-Gram, Global Context,
and Multilingual, on various word similarity
tasks. \newcite{rothe-15-acl} also describe
how standard word vectors can be extended to various data types in
semantic lexicons, e.g., synsets and lexemes in WordNet.

Most of the standard word vector training algorithms use co-occurrence
within window-based contexts to measure relatedness among
words. Several studies question the limitations of defining
relatedness in this way and investigate if the word co-occurrence
matrix can be constructed to encode prior knowledge directly to
improve the quality of word vectors. \newcite{wang-hirst-mohamed-2015}
investigate the notion of relatedness in embedding models by
incorporating syntactic and lexicographic knowledge. In spectral
learning, \newcite{yih-12-lsa} augment the word co-occurrence matrix on
which LSA operates with relational information such that synonyms
will tend to have positive cosine similarity, and antonyms will tend
to have negative similarities.  Their vector space representation
successfully projects synonyms and antonyms on opposite sides in the
projected space. \newcite{chang-13-lsa} further generalize this
approach to encode multiple relations (and not just opposing
relations, such as synonyms and antonyms) using multi-relational LSA.

\ignore{
Our approach of encoding prior knowledge (described in
\S\ref{section:cca-prior-knowledge}) falls under the first thread.
In that we seek to
incorporate prior knowledge in the derivation of CCA based eigenword
embeddings \cite{dhillon-15}. The main motivation behind our approach is
to improve the optimization of correlation between the two views
by weighing them using the external
source of prior knowledge. Like \newcite{faruqui:2015:Retro} and
\newcite{rothe-15-acl}, the prior knowledge is based on lexical
resources such as WordNet, FrameNet and the Paraphrase Database.  We
evaluate our word embeddings on a myriad of datasets and compare them
against standard word vectors such as Glove, Skip-Gram, Global
Context, Multilingual and Eigen and retrofitted-standard word vectors
using the refining techniques of \newcite{faruqui:2015:Retro}.
}

In spectral learning, most of the studies on incorporating prior
knowledge in word vectors focus on LSA based word embeddings
\cite{yih-12-lsa,chang-13-lsa,turney-05,turney-06,turney-10}.

From the technical perspective, our work is also related to that
of \newcite{jagarlamudi-11}, who showed how to generalize CCA so
that it uses locality preserving projections \cite{niyogi2004locality}.
They also assume the existence of a weight matrix in a multi-view
setting that describes the distances between pairs of points in the two
views.

More generally, CCA is an important component for spectral learning
algorithms in the unsupervised setting and with latent variables
\cite{cohen-14c,narayan-16b,stratos-16}. Our method for incorporating prior knowledge into CCA could
potentially be transferred to these algorithms.

\section{Conclusion}

We described a method for incorporating prior knowledge into CCA.
Our method requires a relatively simple change
to the original canonical correlation analysis, where extra counts are
added to the matrix on which singular value decomposition is
performed.  We used our method to derive word embeddings in the style
of eigenwords, and tested them on a set of datasets. Our results
demonstrate several advantages of encoding prior knowledge into
eigenword embeddings.

\section*{Acknowledgements}

The authors would like to thank Paramveer Dhillon for his help with running the \texttt{SWELL}
package. The authors would also like to thank Manaal Faruqui and
Sujay Kumar Jauhar for their help and technical assistance with the
retrofitting package and the word embedding evaluation suite.
Thanks also to Ankur Parikh for early discusions on this project.
This work was completed while the first author
was an intern at the University of Edinburgh, as part of the Equate Scotland program.
This research was supported by an EPSRC grant (EP/L02411X/1) and an EU H2020
grant (688139/H2020-ICT-2015; SUMMA).

\section*{Appendix A: Proofs}

{\small

\begin{proof}[Proof of Lemma 1]
  The proof is similar to the one that appears in
  \newcite{koren2003visualization} for Lemma 3.1. The only difference
  is the use of two views.  Note that $[X^{\top} L Y]_{ij} =
  \sum_{k,k'} X_{ki}L_{kk'}Y_{k'j}$.  As such,
  \begin{align}
    \MoveEqLeft {[X^{\top} L Y]_{ij} = \sum_{k,k'} (n \delta_{kk'} - 1) X_{ki}Y_{k'j} }\\
    & = \sum_{k=1}^n nX_{ki}Y_{kj} - \underbrace{\left( \sum_{k=1}^n X_{ki} \right)}_{0} \times \underbrace{\left( \sum_{k'=1}^n Y_{k'j} \right)}_{0} \\
    & = n[X^{\top}Y]_{ij},
  \end{align}
  \noindent where $\delta_{kk'} = 1$ iff $k= k'$ and $0$ otherwise, and
  the second equality relies on the assumption of the data being
  centered.
\end{proof}

\begin{proof}[Proof of Lemma 2]
Without loss of generality, assume $d \le d'$. Let $u'_1,\ldots,u'_d$ be the left singular vectors of $A$ and $v'_1,\ldots,v'_{d'}$ be the
right ones, and $\sigma_1,\ldots,\sigma_d$ be the singular values. Therefore $A = \sum_{j=1}^d \sigma_j u'_j (v'_j)^{\top}$. In addition, the objective
equals (after substituting $A$):

  \begin{align}
    \MoveEqLeft \sum_{i=1}^m \sum_{j=1}^d \sigma_j \langle u_i, u'_j \rangle \langle v_i, v'_j \rangle = \sum_{j=1}^d \sigma_j \left( \sum_{i=1}^m  \langle u_i, u'_j \rangle \langle v_i, v'_j \rangle \right) \label{eq:CC}
  \end{align}

Note that by the Cauchy-Schwartz inequality: 

\begin{align}
\MoveEqLeft \sum_{j=1}^d \sum_{i=1}^m  \langle u_i, u'_j \rangle \langle v_i, v'_j \rangle = \sum_{i=1}^m \sum_{j=1}^d  \langle u_i, u'_j \rangle \langle v_i, v'_j \rangle \\
& \le \sum_{i=1}^m \sqrt{\sum_{j=1}^d |\langle u_i, u'_j \rangle|^2} \sqrt{\sum_{j=1}^d |\langle v_i, v'_j \rangle|^2} \le m
\end{align}

In addition, note that if we choose $u_i = u'_i$ and $v_i = v'_i$, then the inequality above becomes an equality, and in addition,
the objective in Eq.~\refeq{eq:CC} will equal the sum of the $m$ largest singular vectors $\sum_{j=1}^m \sigma_j$. As such, this assignment
to $u_i$ and $v_i$ maximizes the objective.
\end{proof}

\begin{proof}[Proof of Lemma 3]
  First, by definition of matrix multiplication,
  \begin{align}
    \sum_{k=1}^m (X u_k)^{\top} L \left( Y v_k \right) = \sum_{i,j} L_{ij} \left( \sum_{k=1}^m [Xu_k]_i [Yv_k]_j \right). \label{eq:AA}
  \end{align}
  Also,
  \begin{align}
    \left(d^m_{ij}\right)^2 = \frac{1}{2} \left( \sum_{k=1}^m [Xu_k]_i^2 - 2[Xu_k]_i[Yv_k]_j + [Yv_k]^2_j \right).
  \end{align}
  Therefore, 
  \begin{align}
    \MoveEqLeft{2\sum_{i, j} -L_{ij}\left( d^m_{ij} \right)^2} \\
    & = \sum_{i, j} -L_{ij} \left( \sum_{k=1}^m  - 2[Xu_k]_i[Yv_k]_j \right) \\
    & \,\,\,\,\,\,+ \underbrace{\sum_{i, j} -L_{ij} \left( \sum_{k=1}^m [Xu_k]_i^2 + [Yv_k]^2_j \right)}_{0} \\
    & = 2 \sum_{i,j} L_{ij} \left( \sum_{k=1}^m [Xu_k]_i[Yv_k]_j, \right) \label{eq:BB}
  \end{align}
  \noindent where the first two terms disappear because of the
  definition of the Laplacian. The comparison of Eq.~\refeq{eq:AA} to
  Eq.~\refeq{eq:BB} gives us the necessary result.
\end{proof}
}



\bibliographystyle{acl2012}
\bibliography{nlp}

\begin{thebibliography}{}

\bibitem[\protect\citename{Agirre \bgroup et al.\egroup }2009]{agirre2009study}
Eneko Agirre, Enrique Alfonseca, Keith Hall, Jana Kravalova, Marius
  Pa{\c{s}}ca, and Aitor Soroa.
\newblock 2009.
\newblock A study on similarity and relatedness using distributional and
  wordnet-based approaches.
\newblock In {\em Proceedings of HLT-NAACL}.

\bibitem[\protect\citename{Bach and Jordan}2005]{bach2005probabilistic}
Francis Bach and Michael Jordan.
\newblock 2005.
\newblock {A probabilistic interpretation of canonical correlation analysis}.
\newblock Tech Report 688, Department of Statistics, University of California,
  Berkeley.

\bibitem[\protect\citename{Baker \bgroup et al.\egroup
  }1998]{baker-98-framenet}
Collin~F. Baker, Charles~J. Fillmore, and John~B. Lowe.
\newblock 1998.
\newblock The {Berkeley} {FrameNet} project.
\newblock In {\em Proceedings of {ACL}}.

\bibitem[\protect\citename{Baker \bgroup et al.\egroup
  }2014]{baker2014unsupervised}
Simon Baker, Roi Reichart, and Anna Korhonen.
\newblock 2014.
\newblock An unsupervised model for instance level subcategorization
  acquisition.
\newblock In {\em Proceedings of EMNLP}.

\bibitem[\protect\citename{Bansal \bgroup et al.\egroup }2014]{bansal-14}
Mohit Bansal, Kevin Gimpel, and Karen Livescu.
\newblock 2014.
\newblock Tailoring continuous word representations for dependency parsing.
\newblock In {\em Proceedings of ACL}.

\bibitem[\protect\citename{Belkin \bgroup et al.\egroup
  }2006]{belkin2006manifold}
Mikhail Belkin, Partha Niyogi, and Vikas Sindhwani.
\newblock 2006.
\newblock Manifold regularization: A geometric framework for learning from
  labeled and unlabeled examples.
\newblock {\em Journal of Machine Learning Research}, 7:2399--2434.

\bibitem[\protect\citename{Bengio \bgroup et al.\egroup }2003]{bengio-03}
Yoshua Bengio, R{\'e}jean Ducharme, Pascal Vincent, and Christian Janvin.
\newblock 2003.
\newblock A neural probabilistic language model.
\newblock {\em Journal of Machine Learning Research}, 3:1137--1155.

\bibitem[\protect\citename{Bian \bgroup et al.\egroup }2014]{bian-14}
Jiang Bian, Bin Gao, and Tie-Yan Liu.
\newblock 2014.
\newblock Knowledge-powered deep learning for word embedding.
\newblock In {\em Machine Learning and Knowledge Discovery in Databases},
  volume 8724 of {\em Lecture Notes in Computer Science}, pages 132--148.

\bibitem[\protect\citename{Blum and Mitchell}1998]{blum1998combining}
Avrim Blum and Tom Mitchell.
\newblock 1998.
\newblock Combining labeled and unlabeled data with co-training.
\newblock In {\em Proceedings of {COLT}}.

\bibitem[\protect\citename{Bruni \bgroup et al.\egroup
  }2014]{bruni2014multimodal}
Elia Bruni, Nam-Khanh Tran, and Marco Baroni.
\newblock 2014.
\newblock Multimodal distributional semantics.
\newblock {\em Journal of Artificial Intelligence Research}, 49:1--47.

\bibitem[\protect\citename{Chang \bgroup et al.\egroup }2013]{chang-13-lsa}
Kai-Wei Chang, Wen-tau Yih, and Christopher Meek.
\newblock 2013.
\newblock Multi-relational latent semantic analysis.
\newblock In {\em Proceedings of EMNLP}.

\bibitem[\protect\citename{Chen and de Melo}2015]{chen2015semantic}
Jiaqiang Chen and Gerard de~Melo.
\newblock 2015.
\newblock Semantic information extraction for improved word embeddings.
\newblock In {\em Proceedings of {NAACL} Workshop on Vector Space Modeling for
  {NLP}}.

\bibitem[\protect\citename{Cohen \bgroup et al.\egroup }2014]{cohen-14c}
Shay~B. Cohen, K.~Stratos, Michael Collins, Dean~P. Foster, and Lyle Ungar.
\newblock 2014.
\newblock Spectral learning of latent-variable {PCFGs}: Algorithms and sample
  complexity.
\newblock {\em Journal of Machine Learning Research}.

\bibitem[\protect\citename{Collobert and Weston}2008]{collobert-08}
Ronan Collobert and Jason Weston.
\newblock 2008.
\newblock A unified architecture for natural language processing: Deep neural
  networks with multitask learning.
\newblock In {\em Proceedings of {ICML}}.

\bibitem[\protect\citename{Deerwester \bgroup et al.\egroup
  }1990]{deerwester-90-lsa}
Scott Deerwester, Susan~T. Dumais, George~W. Furnas, Thomas~K. Landauer, and
  Richard Harshman.
\newblock 1990.
\newblock Indexing by latent semantic analysis.
\newblock {\em Journal of the American Society for Information Science},
  41(6):391--407.

\bibitem[\protect\citename{Dhillon \bgroup et al.\egroup }2015]{dhillon-15}
Paramveer~S. Dhillon, Dean~P. Foster, and Lyle~H. Ungar.
\newblock 2015.
\newblock Eigenwords: Spectral word embeddings.
\newblock {\em Journal of Machine Learning Research}, 16:3035--3078.

\bibitem[\protect\citename{Faruqui and Dyer}2014]{faruqui14multi}
Manaal Faruqui and Chris Dyer.
\newblock 2014.
\newblock Improving vector space word representations using multilingual
  correlation.
\newblock In {\em Proceedings of EACL}.

\bibitem[\protect\citename{Faruqui and Dyer}2015]{faruqui:2015}
Manaal Faruqui and Chris Dyer.
\newblock 2015.
\newblock Non-distributional word vector representations.
\newblock In {\em Proceedings of ACL}.

\bibitem[\protect\citename{Faruqui \bgroup et al.\egroup
  }2015]{faruqui:2015:Retro}
Manaal Faruqui, Jesse Dodge, Sujay~K. Jauhar, Chris Dyer, Eduard Hovy, and
  Noah~A. Smith.
\newblock 2015.
\newblock Retrofitting word vectors to semantic lexicons.
\newblock In {\em Proceedings of NAACL}.

\bibitem[\protect\citename{Finkelstein \bgroup et al.\egroup
  }2002]{finkelstein2002placing}
Lev Finkelstein, Gabrilovich Evgenly, Matias Yossi, Rivlin Ehud, Solan Zach,
  Wolfman Gadi, and Ruppin Eytan.
\newblock 2002.
\newblock Placing search in context: The concept revisited.
\newblock {\em ACM Transactions on Information Systems}, 20(1):116--131.

\bibitem[\protect\citename{Fried and Duh}2015]{fried-14}
Daniel Fried and Kevin Duh.
\newblock 2015.
\newblock Incorporating both distributional and relational semantics in word
  representations.
\newblock In {\em Proceedings of ICLR}.

\bibitem[\protect\citename{Ganitkevitch \bgroup et al.\egroup
  }2013]{ganitkevitch-13-ppdb}
Juri Ganitkevitch, Benjamin Van~Durme, and Chris Callison-Burch.
\newblock 2013.
\newblock {PPDB}: The paraphrase database.
\newblock In {\em Proceedings of {NAACL}}.

\bibitem[\protect\citename{Halawi \bgroup et al.\egroup }2012]{halawi2012large}
Guy Halawi, Gideon Dror, Evgeniy Gabrilovich, and Yehuda Koren.
\newblock 2012.
\newblock Large-scale learning of word relatedness with constraints.
\newblock In {\em Proceedings of ACM SIGKDD}.

\bibitem[\protect\citename{Harris}1957]{harris1957co}
Zellig~S. Harris.
\newblock 1957.
\newblock Co-occurrence and transformation in linguistic structure.
\newblock {\em Language}, 33(3):283--340.

\bibitem[\protect\citename{He and Niyogi}2004]{niyogi2004locality}
Xiaofei He and Partha Niyogi.
\newblock 2004.
\newblock Locality preserving projections.
\newblock In {\em Proceedings of {NIPS}}.

\bibitem[\protect\citename{Hill \bgroup et al.\egroup
  }2015]{DBLP:journals/corr/HillRK14}
Felix Hill, Roi Reichart, and Anna Korhonen.
\newblock 2015.
\newblock {SimLex-999}: Evaluating semantic models with (genuine) similarity
  estimation.
\newblock {\em Computational Linguistics}, 41(4):665--695.

\bibitem[\protect\citename{Huang \bgroup et al.\egroup
  }2012]{huang2012improving}
Eric~H Huang, Richard Socher, Christopher~D Manning, and Andrew~Y Ng.
\newblock 2012.
\newblock Improving word representations via global context and multiple word
  prototypes.
\newblock In {\em Proceedings of ACL}.

\bibitem[\protect\citename{Jagarlamudi and
  Daum{\'e}}2012]{jagarlamudi2012regularized}
Jagadeesh Jagarlamudi and Hal Daum{\'e}.
\newblock 2012.
\newblock Regularized interlingual projections: {E}valuation on multilingual
  transliteration.
\newblock In {\em Proceedings of EMNLP-CoNLL}.

\bibitem[\protect\citename{Jagarlamudi \bgroup et al.\egroup
  }2011]{jagarlamudi-11}
Jagadeesh Jagarlamudi, Raghavendra Udupa, and Hal Daum\'{e}.
\newblock 2011.
\newblock Generalization of {CCA} via spectral embedding.
\newblock In {\em Proceedings of the Snowbird Learning Workshop of AISTATS}.

\bibitem[\protect\citename{Koren and Carmel}2003]{koren2003visualization}
Yehuda Koren and Liran Carmel.
\newblock 2003.
\newblock Visualization of labeled data using linear transformations.
\newblock In {\em Proceedings of IEEE Conference on Information Visualization}.

\bibitem[\protect\citename{Landauer \bgroup et al.\egroup }1998]{landauer-98}
Thomas~K. Landauer, Peter~W. Foltz, and Darrell Laham.
\newblock 1998.
\newblock An introduction to latent semantic analysis.
\newblock {\em Discourse Processes}, 25:259--284.

\bibitem[\protect\citename{Lazaridou \bgroup et al.\egroup }2013]{lazaridou-13}
Angeliki Lazaridou, Eva~Maria Vecchi, and Marco Baroni.
\newblock 2013.
\newblock Fish transporters and miracle homes: How compositional distributional
  semantics can help {NP} parsing.
\newblock In {\em Proceedings of {EMNLP}}.

\bibitem[\protect\citename{Mikolov \bgroup et al.\egroup }2013a]{mikolov-13}
Tomas Mikolov, Kai Chen, Greg Corrado, and Jeffrey Dean.
\newblock 2013a.
\newblock Efficient estimation of word representations in vector space.
\newblock In {\em Proceedings of ICLR Workshop}.

\bibitem[\protect\citename{Mikolov \bgroup et al.\egroup }2013b]{mikolov-13a}
Tomas Mikolov, Ilya Sutskever, Kai Chen, Greg~S Corrado, and Jeff Dean.
\newblock 2013b.
\newblock Distributed representations of words and phrases and their
  compositionality.
\newblock In {\em Proceedings of NIPS}.

\bibitem[\protect\citename{Mikolov \bgroup et al.\egroup }2013c]{mikolov-13b}
Tomas Mikolov, Wen tau Yih, and Geoffrey Zweig.
\newblock 2013c.
\newblock Linguistic regularities in continuous space word representations.
\newblock In {\em Proceedings of NAACL-HLT}.

\bibitem[\protect\citename{Miller and
  Charles}1991]{doi:10.1080/01690969108406936}
George~A. Miller and Walter~G. Charles.
\newblock 1991.
\newblock Contextual correlates of semantic similarity.
\newblock {\em Language and Cognitive Processes}, 6(1):1--28.

\bibitem[\protect\citename{Miller}1995]{miller1995wordnet}
George~A Miller.
\newblock 1995.
\newblock {WordNet}: A lexical database for {English}.
\newblock {\em Communications of the ACM}, 38(11):39--41.

\bibitem[\protect\citename{Mnih and Hinton}2007]{mnih-07}
Andriy Mnih and Geoffrey Hinton.
\newblock 2007.
\newblock Three new graphical models for statistical language modelling.
\newblock In {\em Proceedings of ICML}.

\bibitem[\protect\citename{Narayan and Cohen}2016]{narayan-16b}
Shashi Narayan and Shay~B. Cohen.
\newblock 2016.
\newblock Optimizing spectral learning for parsing.
\newblock In {\em Proceedings of {ACL}}.

\bibitem[\protect\citename{Parikh \bgroup et al.\egroup }2014]{parikh-14}
Ankur~P. Parikh, Shay~B. Cohen, and Eric Xing.
\newblock 2014.
\newblock Spectral unsupervised parsing with additive tree metrics.
\newblock In {\em Proceedings of {ACL}}.

\bibitem[\protect\citename{Pennington \bgroup et al.\egroup
  }2014]{pennington2014glove}
Jeffrey Pennington, Richard Socher, and Christopher Manning.
\newblock 2014.
\newblock Glove: Global vectors for word representation.
\newblock In {\em Proceedings of EMNLP}.

\bibitem[\protect\citename{Radinsky \bgroup et al.\egroup
  }2011]{radinsky2011word}
Kira Radinsky, Eugene Agichtein, Evgeniy Gabrilovich, and Shaul Markovitch.
\newblock 2011.
\newblock A word at a time: Computing word relatedness using temporal semantic
  analysis.
\newblock In {\em Proceedings of ACM WWW}.

\bibitem[\protect\citename{Rothe and Sch\"{u}tze}2015]{rothe-15-acl}
Sascha Rothe and Hinrich Sch\"{u}tze.
\newblock 2015.
\newblock {AutoExtend}: Extending word embeddings to embeddings for synsets and
  lexemes.
\newblock In {\em Proceedings of ACL-IJCNLP}.

\bibitem[\protect\citename{Rubenstein and
  Goodenough}1965]{Rubenstein:1965:CCS:365628.365657}
Herbert Rubenstein and John~B. Goodenough.
\newblock 1965.
\newblock Contextual correlates of synonymy.
\newblock {\em Communications of the ACM}, 8(10):627--633.

\bibitem[\protect\citename{Silberer \bgroup et al.\egroup
  }2013]{silberer2013models}
Carina Silberer, Vittorio Ferrari, and Mirella Lapata.
\newblock 2013.
\newblock Models of semantic representation with visual attributes.
\newblock In {\em Proceedings of ACL}.

\bibitem[\protect\citename{Socher \bgroup et al.\egroup }2013]{socher-13-cvg}
Richard Socher, John Bauer, Christopher~D. Manning, and Andrew~Y. Ng.
\newblock 2013.
\newblock Parsing with compositional vector grammars.
\newblock In {\em Proceedings of ACL}.

\bibitem[\protect\citename{Stratos \bgroup et al.\egroup
  }2015]{stratos2015model}
Karl Stratos, Michael Collins, and Daniel Hsu.
\newblock 2015.
\newblock Model-based word embeddings from decompositions of count matrices.
\newblock In {\em Proceedings of {ACL}}.

\bibitem[\protect\citename{Stratos \bgroup et al.\egroup }2016]{stratos-16}
Karl Stratos, Michael Collins, and Daniel Hsu.
\newblock 2016.
\newblock Unsupervised part-of-speech tagging with anchor hidden markov models.
\newblock {\em Transactions of the Association for Computational Linguistics},
  4:245--257.

\bibitem[\protect\citename{Turney and Littman}2005]{turney-05}
Peter~D. Turney and Michael~L. Littman.
\newblock 2005.
\newblock Corpus-based learning of analogies and semantic relations.
\newblock {\em Machine Learning}, 60(1-3):251--278.

\bibitem[\protect\citename{Turney and Pantel}2010]{turney-10}
Peter~D. Turney and Patrick Pantel.
\newblock 2010.
\newblock From frequency to meaning: Vector space models of semantics.
\newblock {\em Journal of Artificial Intelligence Research}, 37(1):141--188.

\bibitem[\protect\citename{Turney}2006]{turney-06}
Peter~D. Turney.
\newblock 2006.
\newblock Similarity of semantic relations.
\newblock {\em Computational Linguistics}, 32(3):379--416.

\bibitem[\protect\citename{Wang \bgroup et al.\egroup }2014]{wang-14}
Zhen Wang, Jianwen Zhang, Jianlin Feng, and Zheng Chen.
\newblock 2014.
\newblock Knowledge graph and text jointly embedding.
\newblock In {\em Proceedings of EMNLP}.

\bibitem[\protect\citename{Wang \bgroup et al.\egroup
  }2015]{wang-hirst-mohamed-2015}
Tong Wang, Abdelrahman Mohamed, and Graeme Hirst.
\newblock 2015.
\newblock Learning lexical embeddings with syntactic and lexicographic
  knowledge.
\newblock In {\em Proceedings of ACL-IJCNLP}.

\bibitem[\protect\citename{Xu \bgroup et al.\egroup }2014]{Xu-14}
Chang Xu, Yalong Bai, Jiang Bian, Bin Gao, Gang Wang, Xiaoguang Liu, and
  Tie-Yan Liu.
\newblock 2014.
\newblock {RC-NET}: A general framework for incorporating knowledge into word
  representations.
\newblock In {\em Proceedings of the ACM CIKM}.

\bibitem[\protect\citename{Yang and Powers}2005]{yang2005measuring}
Dongqiang Yang and David~MW Powers.
\newblock 2005.
\newblock Measuring semantic similarity in the taxonomy of {WordNet}.
\newblock In {\em Proceedings of the Australasian Conference on Computer
  Science}.

\bibitem[\protect\citename{Yarowsky}1995]{yarowsky1995unsupervised}
David Yarowsky.
\newblock 1995.
\newblock Unsupervised word sense disambiguation rivaling supervised methods.
\newblock In {\em Proceedings of ACL}.

\bibitem[\protect\citename{Yih \bgroup et al.\egroup }2012]{yih-12-lsa}
Wen-tau Yih, Geoffrey Zweig, and John Platt.
\newblock 2012.
\newblock Polarity inducing latent semantic analysis.
\newblock In {\em Proceedings of EMNLP-CoNLL}.

\bibitem[\protect\citename{Yu and Dredze}2014]{yu-dredze-2014}
Mo~Yu and Mark Dredze.
\newblock 2014.
\newblock Improving lexical embeddings with semantic knowledge.
\newblock In {\em Proceedings of ACL}.

\end{thebibliography}

\end{document}